\documentclass{article}

\usepackage{natbib}
\usepackage[utf8]{inputenc} 
\usepackage[T1]{fontenc}    
\usepackage{hyperref}       
\usepackage{url}            
\usepackage{amsfonts}       

\usepackage{amsmath,amsthm}
\usepackage{algorithm}
\usepackage{algorithmic}
\usepackage{color}
\usepackage{graphicx}

\newcommand{\E}{\ensuremath{\mathbf{E}}}

\newcommand{\I}{\ensuremath{\mathbf{I}}}

\newcommand{\M}{\ensuremath{\mathbf{M}}}

\newcommand{\U}{\ensuremath{\mathbf{U}}}
\newcommand{\V}{\ensuremath{\mathbf{V}}}

\newcommand{\p}{\ensuremath{\mathbf{p}}}
\newcommand{\q}{\ensuremath{\mathbf{q}}}

\newcommand{\uu}{\ensuremath{\mathbf{u}}}
\newcommand{\vv}{\ensuremath{\mathbf{v}}}
\newcommand{\w}{\ensuremath{\mathbf{w}}}
\newcommand{\x}{\ensuremath{\mathbf{x}}}
\newcommand{\y}{\ensuremath{\mathbf{y}}}
\newcommand{\z}{\ensuremath{\mathbf{z}}}
\newcommand{\0}{\ensuremath{\mathbf{0}}}

\newcommand{\bbE}{\ensuremath{\mathbb{E}}}

\newcommand{\bbR}{\ensuremath{\mathbb{R}}}

\newcommand{\calC}{\ensuremath{\mathcal{C}}}

\newcommand{\calO}{\ensuremath{\mathcal{O}}}

\newcommand{\norm}[1]{\left\lVert#1\right\rVert}

{%
\begin{list}{#1}{
\vspace{-\topsep}
\vspace{-\partopsep}
\setlength{\itemindent}{0cm}
\setlength{\rightmargin}{0cm}
\setlength{\listparindent}{0cm}
\settowidth{\labelwidth}{#1}
\setlength{\leftmargin}{\labelwidth}
\addtolength{\leftmargin}{\labelsep}
\setlength{\itemsep}{0cm}
}%
}%
{%
\end{list}
\vspace{-\topsep}
\vspace{-\partopsep}
}

{\begin{enumerate}%
}%
{\end{enumerate}}

\newcommand{\traceop}{\operatorname{tr}}
\newcommand{\trace}[1]{\ensuremath{\traceop\left(#1\right)}}

\newtheorem{thm}{Theorem}
\newtheorem{lem}[thm]{Lemma}

\newtheorem{rmk}[thm]{Remark}

\DeclareMathOperator*{\argmin}{arg\,min}

\newcommand{\ie}{i.e.\@}
\newcommand{\eg}{e.g.\@}
\newcommand{\sbr}[1]{\left[#1\right]}
\newcommand{\rbr}[1]{\left(#1\right)}
\newcommand{\cbr}[1]{\left\{#1\right\}}
\newcommand{\dotp}[2]{\langle{#1},\,{#2}\rangle}
\newcommand{\tx}{\tilde{\x}}
\newcommand{\tz}{\tilde{\z}}
\newcommand{\tU}{\tilde{\U}}
\newcommand{\tV}{\tilde{\V}}

\newcommand{\hv}{\hat{\vv}}
\newcommand{\hw}{\hat{\w}}
\newcommand{\hL}{\hat{L}}
\newcommand{\hF}{\hat{F}}
\newcommand{\hG}{\hat{G}}
\newcommand{\tw}{\tilde{\w}}
\newcommand{\tv}{\tilde{\vv}}

\newcommand{\Exz}{\E_{xz}}
\newcommand{\Exx}{\E_{xx}}
\newcommand{\Ezz}{\E_{zz}}

\usepackage{mathtools}

\begin{document}

\title{Everything old is new again: \\
A multi-view learning approach to learning using privileged information and distillation}
\author{Weiran Wang \\
weiranw@amazon.com}
\maketitle

\begin{abstract}
We adopt a multi-view approach for analyzing two knowledge transfer 
settings---learning using privileged information (LUPI) and distillation---in
a common framework. 
Under reasonable assumptions about the complexities of hypothesis spaces,
and being optimistic about the expected loss achievable by the student (in
distillation) and a transformed teacher predictor (in LUPI), 
we show that encouraging agreement between the
teacher and the student leads to reduced search space. As a result,
improved convergence rate can be obtained with regularized empirical risk minimization.
\end{abstract}

\section{Introduction}
\label{sec:intro}

Transferring knowledge learned by a powerful model (``teacher'') to a simpler model (``student'') has become a theme in machine learning. The goal of the knowledge transfer is to have the teacher guide the learning process of the student, so as to achieve high prediction accuracy, or to reduce the sample complexity, which are otherwise hard for the student to achieve by itself. This learning paradigm is practically useful when it is necessary to deploy simpler models to real-world systems, which requires small memory footage or fast processing time. 

We focus on two specific settings of knowledge transfer in this work. The first one is \emph{learning using privileged information (LUPI)}~\citep{VapnikVashis09a}, in which the teacher provides an additional set of feature representation to the student during its training process but not the test time, and the extra feature set contains richer information to make the learning problem easier for the student; an example is that the ``student may normally only have access to the image of a biopsy to predict the existence of cancer, but during the training process, it also has access to the medical report of an oncologist''~\citep{Lopez_15a}. 
The second setting is \emph{distillation}~\citep{BaCaruan14a,Hinton_15a}, in which the teacher and the student have access to the same samples (and the same features), and the teacher learns a complex model (\eg, deep neural networks), which provides soft targets (\eg, the decision values) for the student (\eg, a shallow neural network) to mimic directly during its training process. 
\citet{Lopez_15a}  have unified the two settings as \emph{generalized distillation}, which facilitates learning from multiple machines and data representations. We take the same unified view as~\citet{Lopez_15a} in this work and analyze different settings in a common learning framework, while using intuitions and assumptions specific to LUPI or distillation when appropriate. 
While in LUPI the learning task is often supervised, the constraint that privileged information is only available during training resembles that of (unsupervised) multi-view representation learning~\citep{Wang_15b}, where the goal is to learn feature transformations from multiple measurements (``views'') of the input data that can be useful for downstream tasks (which are often supervised).

With the intuition of the teacher making the learning problem easier for the student, the LUPI setting is theoretically motivated~\citep{VapnikVashis09a,PechyonVapnik10a}. For example, for binary supervised learning (with the 0-1 loss), when $n$ training samples are used, the convergence rate is $\calO \rbr{ \frac{1}{n} }$ in the separable case (\ie, when there exists a perfect classifier), and the rate degrades to $\calO \rbr{ \frac{1}{\sqrt{n}} }$ when the problem is not.  However, when learning an SVM classifier in the non-separable case, if the teacher can supply the student with the slack variables associated with the optimal solution for each sample (thus setting the correct limit on each sample), the learning objective becomes separable for the student and the faster $\calO \rbr{\frac{1}{n}}$ rate can be restored. It has been hypothesized that faster convergence rate can in general be obtained under LUPI~\citep{Lopez_15a}, although it is not clear under what conditions and using what learning method this is achievable.

\paragraph{Our contributions} 
We can consider the normal feature set and the privileged information as two ``views'' of the input data.
There exists a long line of (both theoretical and empirical) research studying the multi-view learning setting, supervised or not. And the most popular intuition in multi-view learning is that the views need to agree, either on the features~\citep{HermanBlunsom14a,Wang_15b,Fang_15a} or on the predictions~\citep{BlumMitchel98a,Sindhw_05b,RosenbBartlet07a,KakadeFoster07a,BalcanBlum10a,BlumMansour17a}. We heavily draw inspirations from these prior work and show how, under reasonable assumptions on the model complexities for distillation and LUPI, agreement between the views help reduce the search space for the student and therefore leads to faster convergence. 

In particular, we analyze distillation and LUPI in the framework of learning linear predictors with convex, smooth and non-negative losses, which are commonly used in machine learning.
We show that improved rate is possible by being \emph{optimistic}, that is, assuming the student \emph{can} achieve low expected loss for distillation, and assuming a good predictor on the privileged information can be transformed into a good predictor on the regular feature set for LUPI. Under these assumptions, we perform regularized ERM where the regularization term measures the prediction discrepancy (squared difference) between the student and the teacher; this is an intuitive and easy to optimize regularizer and often used in practice. The solution to regularized ERM achieves faster convergence than the one without regularization, or equivalently, there exists a larger range of optimal loss for which the student achieves the $\calO \rbr{\frac{1}{n}}$ convergence.
Interestingly, in the multi-view distillation and LUPI settings, the complexity control is achieved through the coordinate system defined by Canonical Correlation Analysis~\citep{Hotell36a}.  

\paragraph{Notations} Bold lower case letters (such as $\uu$, $\vv$) denote vectors, and bold upper case letters (such as $\U$, $\V$) denote matrices. For vectors, subscripts index samples of a random vector, while superscripts index the coordinates.
$[\vv^i]_{i=1}^d$ denote the $d$-dimensional vector whose $i$-th
coordinate is $\vv^i$. We use $\0$ and $\I$ to denote the all-zero vector
and identity matrix respectively, whose dimensions can be inferred from the
context. 
A convex function $f(\w)$ is $\beta$-smooth in $\w$
if $f(\w_1) \le f(\w_2) + \nabla f(\w_2)^\top (\w_1 - \w_2) +
\frac{\beta}{2} \norm{\w_1 - \w_2}^2$, $\forall \w_1, \w_2$, and it is
$\sigma$-strongly convex in $\w$ if 
$f(\w_1) \ge f(\w_2) + \nabla f(\w_2)^\top (\w_1 - \w_2) +
\frac{\sigma}{2} \norm{\w_1 - \w_2}^2$, $\forall \w_1, \w_2$.

\subsection{A brief review on CCA}
\label{sec:cca-review}

We briefly review Canonical Correlation Analysis (CCA,~\citealp{Hotell36a}), a classical method for measuring the correlation between two random vectors, as it plays an important role later.  
Let $\x \in \bbR^{d_x}$ and $\z \in \bbR^{d_z}$ be two random vectors with
a joint probability distribution $P(\x,\z)$. The simultaneous
formulation of CCA finds a set of $r$ directions
(canonical directions) for each view, collected in matrices $\U \in \bbR^{d_x \times r}$ and $\V\in \bbR^{d_z \times r}$, such that projections of $(\x,\y)$ onto these directions are maximally correlated:
\begin{gather} \label{e:cca-population}
  \max_{\U,\V} \;  \trace{ \U^\top \Exz \V}   \quad  \text{s.t.} \quad \U^\top \Exx \U =\V^\top \Ezz \V = \I 
\end{gather}
where the cross- and auto-covariance matrices are defined as 
\begin{gather} \label{e:cov-population}
\Exz = \bbE [ \x \z^\top],  \quad \Exx = \bbE [ \x \x^\top],  \quad \Ezz = \bbE [ \z \z^\top].
\end{gather}
In this work, we assume for simplicity that the random vectors have zero
mean ($\bbE [\x] = \0$ and $\bbE [\z] = \0$), and identity covariance ($\Exx=\I$ and $\Ezz=\I$).
The global optimum of~\eqref{e:cca-population}, denoted by $(\U^*,\V^*)$, can be computed as follows. 
Let the full SVD of 
$\Exz$ be
\begin{align*}
\Exz = \tU \Sigma \tV^\top, 
\end{align*}
where $\tU \in \bbR^{d_x \times d_x}$ and $\tV \in \bbR^{d_z \times d_z}$ are unitary, and $\Sigma$ contains the singular values (canonical correlations) on its diagonal:
\begin{align*}
1 \ge \lambda_1 \ge \dots \ge \lambda_{rank(\Exz) > 0}.
\end{align*}
Then the optimal directions are
\begin{align*}
(\U^*,\,\V^*)=(\tU_{1:r},\,\V_{1:r} ), 
\end{align*}
where $\U_{1:r}$ denote the submatrix of $\U$ containing the first $r$ columns,
and the optimal objective value is $\sum_{i=1}^r \lambda_i$. 
We refer the readers to~\citet{Gao_17a,ZhuLi16a} for efficient numerical procedures
for computing the solution.

Note that the canonical directions and correlations are derived solely from the
multi-view input distribution $P(\x,\z)$ (regardless of the label
information), and they define a new
coordinate system than that of the input space. As we will see later, based on
certain assumptions, this new system facilitates complexity
control for statistical learning.

\section{Learning with smooth and non-negative loss}
\label{sec:smooth-loss}

We now discuss the learning setup in which we analyze different knowledge
transfer settings. 
Consider linear predictors for supervised learning: given
i.i.d. samples of random variables $(\x,y)$ drawn from an unknown
distribution $D$, we would like to learn a linear predictor $\w$ to
predict the target $y$ from the input $\x$, based on the inner product
$\w^\top \x$, and the discrepancy between the prediction and target is
measured by a instantaneous loss function $\ell(\w^\top \x, y)$. 
Let $L(\w) = \bbE_{D} [\ell(\w^\top \x, y)]$ be the expected loss associated with the predictor $\w$. 
With $n$ i.i.d. samples $\cbr{(\x_1,y_1),\dots, (\x_n,y_n)}$,
the empirical loss is defined as $\hL(\w)=\frac{1}{n} \sum_{i=1}^n
\ell(\w^\top \x_i, y_i)$. 

We assume the loss $\ell(a, b) > 0$ is convex and $\beta$-smooth in the
first argument $a$. Such losses are widely used in machine learning. For example, the least
squares loss $\frac{1}{2} (a - b)^2$ is $1$-smooth in $a$, and the
cross-entropy loss $\log \rbr{1+\exp(-b \cdot a)}$ is $\frac{1}{4}$-smooth in $a$. 
In addition to $\Exx=\I$ (as in Sec~\ref{sec:cca-review}), we further assume 
$\norm{\x} \le R$, implying that $\ell(\w^\top \x, y)$ is $(\beta
R^2)$-smooth in $\w$. Sometimes we also assume the loss $\ell(a,
b)$ to be strongly-convex in $a$; the least squares
loss is $1$-strongly convex, and the logistic loss is strongly convex as
long as the first argument (decision value) is bounded.

The goal of convex-smooth-bounded
learning~\citep[Sec~12.2.2]{ShalevBen-David14a} is to find a good
predictor $\hw$ such that
\begin{align*}
L(\hw) \le \min_{\norm{\w} \le B}\; L (\w) + \epsilon
\end{align*}
where $\cbr{ \w: \norm{\w}\le B }$ is the hypothesis class we would like
to learn, and $\epsilon>0$ is the excess risk.  
Such predictors can be learned (properly) by solving the constrained ERM problem, or improperly by solving the regularized ERM. 
\begin{lem}[\citet{Srebro_12a}]
\label{lem:optimistic}
Let $L^* = \min_{\norm{\w} \le B}\, L (\w)$ be the optimal expected loss.
Then for either of the following estimators
\begin{align*}
& \text{(constrained ERM)} \qquad \hw = \argmin_{\norm{\w} \le B}\; \hL
  (\w) \\
& \text{(regularized ERM)} \qquad \hw = \argmin\; \hL (\w) + \frac{\gamma}{2} \norm{\w}^2
\end{align*}
where $\gamma = \frac{8 \beta R^2}{n} + \sqrt{\frac{64 \beta^2 R^4}{n^2} + \frac{16 \beta R^2 L^*}{n B^2}}$,
we have
\begin{align} \label{e:optimistic-rate}
\bbE [L(\hw) - L^*] = \tilde{\calO} \rbr{ \frac{\beta R^2 B^2}{n} + \sqrt{\frac{\beta R^2 B^2 L^*}{n}} }
\end{align}
where the expectation is over the $n$ training samples.
\end{lem}

\begin{rmk}
\label{rmk:optimistic}
Observe that, for learning with smooth non-negative losses, the rate of
convergence can be bounded using the optimal expected loss. For large
$L^*$, the second term in~\eqref{e:optimistic-rate} is dominant and we
obtain the usual $\tilde{\calO} \rbr{\frac{1}{\sqrt{n}}}$ rate. 
However, if $L^*$ is close to zero and the first term becomes dominant, 
we obtain a faster $\tilde{\calO} \rbr{\frac{1}{n}}$ rate.
This phenomenon is in spirit similar to the faster convergence for $0$-$1$ 
loss when the problem is separable.
Overall, we can bound the sample complexity to achieve $\epsilon$ excess error as 
\begin{align} \label{e:optimistic-sample}
\tilde{\calO} \rbr{\frac{\beta R^2 B^2}{\epsilon} \cdot \frac{\epsilon + L^*}{\epsilon}} .
\end{align} 
We thus see that the transition between the two regimes 
happens at $L^*=\calO (\epsilon)$: if the optimal loss is not
much larger than the target excess error, the sample complexity is 
$\tilde{\calO} \rbr{\frac{\beta R^2 B^2}{\epsilon}}$ (corresponding to
$\tilde{\calO} \rbr{\frac{1}{n}}$ convergence); otherwise the sample
complexity degrades to $\tilde{\calO} \rbr{\frac{\beta R^2 B^2 L^*}{\epsilon^2}}$ 
(corresponding to $\tilde{\calO} \rbr{\frac{1}{\sqrt{n}}}$ convergence). 
This faster $\tilde{\calO} \rbr{\frac{1}{n}}$ rate achieved for small
expected loss is known as the \emph{optimistic rate}, which may not be attainable
by learning with non-smooth losses (\eg, convex-Lipschitz-bounded
learning). 
\end{rmk}

\section{Analysis for distillation}
\label{sec:distill}

In the distillation setting~\citep{BaCaruan14a, Hinton_15a}, we have only
one view $\x$. We first train a powerful, 
teacher model $\hv$, which provides soft supervision for training a
simpler, student model $\hw$.
In the context of learning linear predictors, we let $\hv$ be learned from a larger hypothesis class $\calC_v = \cbr{\w: \norm{\vv}\le B_v}$, while $\w$ is constrained to come from a smaller hypothesis class $\calC_w = \cbr{\w: \norm{\w}\le B_w}$, with $B_w \le B_v$. 

We will assume that the optimal predictor from $\calC_w$, \ie,
$\vv^*=\argmin_{\vv\in\calC_v}\, L(\vv)$, has small expected loss
$L_v^* \approx 0$, so that teacher model, learned with the ERM 
\begin{align*}
  \hv = \argmin_{\vv\in\calC_v}\; \frac{1}{n} \sum_{i=1}^n \ell(\vv^\top
\z_i, y_i)
\end{align*}
enjoys faster convergence:
\begin{align*}
\epsilon_t := \bbE \sbr{L(\hv) - L_v^*} \le \tilde{\calO} \rbr{\frac{1}{n}} .
\end{align*}

The key assumption that allows us to accelerate the learning of the student
model is that the optimal predictor from the smaller hypothesis class $\calC_w$ 
agrees well with the teacher model on \emph{prediction values}, \ie, there
exists some small $S>0$, such that 
\begin{align} \label{e:agreement-assumption-distillation}
\bbE \norm{ {\w^*}^\top \x - {\hv}^\top \x}^2 \le S^2. 
\end{align}

As the lemma below shows, this assumption holds, \eg, when the loss is strongly convex, and the
student model \emph{can} achieve small expected loss $L_w^*$ relative to $L_v^*$.
\begin{lem}\label{lem:agreement}
If the instantaneous loss $\ell(\cdot, \cdot)$ is $\sigma$-strongly convex
in the first argument, we have
\begin{align*}
\bbE \norm{{\w^*}^\top \x - \hv^\top \x }^2 \le \frac{4 (L_w^* - L_v^* + \epsilon_t)}{\sigma}
\end{align*}
where the expectation is taken over the data distribution and the samples
for learning $\hv$.
\end{lem}
\begin{proof}
Let $\w$ be any predictor from $\calC_v$. By the $\sigma$-strong convexity
of $\ell(\cdot, \cdot)$, we have for any $(\x,y)$ that 
\begin{align*}
\ell ({\w}^\top \x, y) - \ell ({\vv^*}^\top \x, y) 
\ge \ell^\prime ({\vv^*}^\top \x, y) \cdot ({\w}^\top \x - {\vv^*}^\top \x) +
  \frac{\sigma}{2}  \norm{{\w}^\top \x - {\vv^*}^\top \x }^2. 
\end{align*}
Taking expectation over $(\x,y)$, we have
\begin{align} \label{e:agreement-strong-convexity}
L ({\w}) - L({\vv^*}) \ge \dotp{\bbE \sbr{ \ell^\prime ({\vv^*}^\top
  \x, y) \cdot \x}}{\w - {\vv^*}} + \frac{\sigma}{2} \cdot \bbE \norm{{\w}^\top \x - {\vv^*}^\top \x }^2.
\end{align}
Note that $\nabla L({\vv^*}) = \bbE \sbr{  \ell^\prime ({\vv^*}^\top
  \x, y) \cdot \x }$. Since both $\w, {\vv^*} \in \calC_v$, and ${\vv^*}$ is the minimizer of
$L(\vv)$ over $\calC_v$, we have by the first order optimality condition
that $\dotp{\nabla L({\vv^*})}{\w - {\vv^*}} \ge 0$. Substituting this
into~\eqref{e:agreement-strong-convexity} yields
\begin{align*}
L ({\w}) - L({\vv^*}) \ge \frac{\sigma}{2} \cdot \bbE \norm{{\w}^\top \x
  - {\vv^*}^\top \x }^2 .
\end{align*}
Setting $\w=\w^*$ and $\w=\hv$ in the above inequality, we obtain respectively 
\begin{align*}
L_w^* - L_v^* & \ge \frac{\sigma}{2} \cdot \bbE \norm{{\w^*}^\top \x -
  {\vv^*}^\top \x }^2,  \\
\epsilon_t = \bbE \sbr{ L ({\hv}) - L_v^* } & \ge \frac{\sigma}{2} \cdot \bbE \norm{{\hv}^\top \x - {\vv^*}^\top \x }^2.
\end{align*}
Combining the above two inequalities, we have
\begin{align*}
\bbE \norm{{\w^*}^\top \x - {\hv}^\top \x }^2
\le 2 \bbE \norm{{\w^*}^\top \x - {\vv^*}^\top \x }^2
+ 2 \bbE \norm{{\hv}^\top \x - {\vv^*}^\top \x }^2
\le \frac{4}{\sigma} (L_w^* - L_v^* + \epsilon_t).
\end{align*}
\end{proof}

Due to the assumption that $\Exx = \I$, the condition~\eqref{e:agreement-assumption-distillation} conveniently reduces to 
\begin{align*}
\norm{\w^* - \hv} \le S.
\end{align*}
Therefore, under the assumption of small discrepancy between the student and the teacher, our search space for $\w^*$ can be further reduced to 
\begin{align*}
\calC_s = \cbr{\w: \norm{\w}\le B_w, \; \norm{\w - \hv} \le S},
\end{align*}
which is much smaller than $\calC_w$ for small $S$.



It is then natural for us to perform regularized ERM to take advantage of
the additional complexity constraint. We propose to use the regularization 
$\bbE \norm{{\w^*}^\top \x - {\hv}^\top \x}^2\le S^2$ which encourages the
student to agree with the teacher on the decision values (soft targets) over the
distribution. In practice, this term can be approximated on large set of
unlabeled data as it does not require ground truth targets $y$. 

\begin{thm} \label{thm:distill}
Compute the following predictor  
\begin{align} \label{e:distill-rerm}
\hw = \argmin_{\w \in \calC_w}\; \frac{1}{n} \sum_{i=1}^n \ell (\w^\top
  \x_i, y_i)  + \frac{\nu}{2} \norm{\w - \hv}^2
\end{align}
with $\nu=\frac{8\beta R^2}{n} + \sqrt{\frac{64\beta^2 R^4}{n^2} +
  \frac{16\beta R^2 L_w^*}{n S^2}}$. 
The we have
\begin{align*}
\bbE \sbr{L(\hw) - L_w^*} \le \frac{16 \beta R^2 S^2}{n} +
  \sqrt{\frac{16\beta R^2 S^2 L_w^*}{n}}.
\end{align*}
\end{thm}
\begin{proof}
The proof is similar to~\citet[Theorem~5]{Srebro_12a}. 
Since $\hw$ is the minimizer to the regularized objective, we have for any
$\w \in \calC_w$  that
\begin{align*}
\hL(\hw) \le 
\hL(\hw) + \frac{\nu}{2} \norm{\hw - \hv}^2 \le
\hL(\w) + \frac{\nu}{2} \norm{\w - \hv}^2  .
\end{align*}
Observe that the regularized ERM is $\nu$-strongly convex in $\w$. Taking
the expectation of the above inequality and applying the stability result
(Lemma~\ref{lem:stability} in appendix), we obtain
\begin{align*}
\bbE \sbr{L (\hw) - L (\w)} \le
  \rbr{\frac{1}{1-\frac{8 \beta R^2}{\nu n}} -1} L(\w) 
+ \frac{1}{1-\frac{8 \beta R^2}{\nu n}} \cdot \frac{\nu \norm{\w - \hv}^2}{2}.
\end{align*}
Setting $\w = \w^*$, and substituting in $\norm{\w^* - \hv} \le S^2$
yields
\begin{align*}
\bbE \sbr{L (\hw) - L (\w)} \le
  \rbr{\frac{1}{1-\frac{8 \beta R^2}{\nu n}} -1} L_w^*
+ \frac{1}{1-\frac{8\beta R^2}{\nu n}} \cdot \frac{\nu S^2}{2}.
\end{align*}
Minimizing the RHS over $\nu$ gives $\nu=\frac{8 \beta R^2}{n} +
\sqrt{\frac{64 \beta^2 R^4}{n^2} + \frac{16 \beta R^2 L_w^*}{n S^2}}$, and 
\begin{align*}
\bbE \sbr{L (\hw) - L (\w)} \le
  \frac{8 \beta R^2 S^2}{n} + \sqrt{\frac{64 \beta^2 R^4 S^4}{n^2} + \frac{16 \beta
 R^2 S^2 L_w^*}{n}} \le \frac{16 \beta R^2 S^2}{n} + \sqrt{\frac{16\beta
  R^2 S^2 L_w^*}{n}}
\end{align*}
where we have used the inequality $\sqrt{a+b} \le \sqrt{a} + \sqrt{b}$ for
$a,b\ge 0$.
\end{proof}

\begin{rmk}
\label{rmk:distill}
As long as $S^2 \ll B^2$, the convergence rate in
Theorem~\ref{thm:distill} is much faster than that of
Lemma~\ref{lem:optimistic}. 
We observe from Lemma~\ref{lem:agreement} that $S^2$ can be of the order
$\calO \rbr{ \frac{L_w^* + \epsilon_t}{\sigma} }$. In this case, the above
theorem leads to the following overall convergence (by our assumption,
$\epsilon_t$-related terms are of higher order):
\begin{align} \label{e:optimistic-distill}
\bbE \sbr{L (\hw) - L_w^*} = \tilde{\calO} \rbr{ \frac{\beta R^2
  L_w^*}{\sigma n}  +  \sqrt{\frac{\beta R^2 (L_w^*)^2}{\sigma n}} }.
\end{align}
Equivalently, we can bound the sample complexity to achieve $\epsilon$ excess error as 
\begin{align}  \label{e:optimistic-distill-sample}
\calO \rbr{\frac{\beta R^2 L_w^*}{\sigma \epsilon} \cdot \frac{\epsilon + L_w^*}{\epsilon}} .
\end{align} 
Compare this rate with the one obtained without the additional complexity
control (Remark~\ref{rmk:optimistic}). 
In~\eqref{e:optimistic-distill-sample}, as long as $L_{w}^* = \calO
(\sqrt{\epsilon})$, rather than requiring $L_{w}^* = \calO (\epsilon)$
as in~\eqref{e:optimistic-sample}, the effective convergence rate of~\eqref{e:distill-rerm} is $\calO
(\frac{1}{n})$.
To see that the difference is significant, if the target excess error is
$\epsilon=10^{-6}$, we now only need $L_w^*$ to be on the order of $10^{-3}$ 
to achieve the optimistic rate,  instead of $10^{-6}$ without
distillation. 
In other words, we achieve the optimistic rate in the distillation setting
with much less stringent condition on the optimal expected loss.
\end{rmk}

\section{Analysis for learning using privileged information}
\label{sec:LUPI}

In the case of learning with privileged information, we have access to
two views: the regular features $\x$ (used by the student) 
and the privileged information $\z$ (used by the teacher).
We first discuss how the correlation between them comes into play 
when transferring knowledge from $\z$ to $\x$. The prerequisites on CCA are
discussed in Section~\ref{sec:cca-review}.

Perform the following change of coordinates (recall that $\tU$ and $\tV$
have full dimensions and thus these new variables are well defined): 
\begin{align}\label{e:cca-coordinate-system}
\begin{aligned}
\tx  & = \tU^\top \x \in \bbR^{d_x}, 
\hspace{4em} \tz  = \tV^\top  \z \in \bbR^{d_z}, \\
\tw & = \tU^\top \w \in \bbR^{d_x},
\hspace{3.8em} \tv  = \tV^\top \vv \in \bbR^{d_z}.
\end{aligned} 
\end{align}
Clearly, the transformed data has identity covariance, \ie, $\bbE \sbr{\tx \tx^\top} = \I$ and $\bbE \sbr{\tz \tz^\top} = \I$.
On the other hand, we have $\tw^\top \tx = \w^\top \x$ and $\tv^\top \tz = \vv^\top \z$.

Based on the above coordinate transformation and the definition of CCA, an important observation made by~\citet{KakadeFoster07a} is the following equality.
\begin{lem}[Lemma~2 of~\citet{KakadeFoster07a} re-stated]
\label{lem:agreement-as-regularizer}
For any $\w \in \bbR^{d_x}$ and $\vv \in \bbR^{d_y}$, we have
\begin{align*}
\bbE \norm{\w^\top \x - \vv^\top \z}^2
& = \sum_{i=1}^{d_x} (1-\lambda_i)  (\tw^i)^2 + \sum_{i=1}^{r} \lambda_i (\tw^i - \tv^i)^2 + \sum_{i=1}^{d_z} (1-\lambda_i) (\tv^i)^2
\\
& =  \sum_{i=1}^{d_x} (\tw^i - \lambda_i \tv^i)^2 + \sum_{i=1}^{d_z} (1-\lambda_i^2) (\tv^i)^2,
\end{align*}
with the convention that $\lambda_i=0$ for $i>rank(\Exz)$, and that $\tv^i=0$ for $i>d_z$.
\end{lem}
This lemma implies, for a pair of predictors $(\w,\vv)$ that agree well on
the decision values, the predictors have low complexity in the CCA
coordinate system: if $\lambda_i$ is large (close to $1$), minimizing the
discrepancy encourages $\tw^i$ to be close to $\tv^i$, and if
$\lambda_i$ is small (close to $0$), minimizing the discrepancy encourages $\tw^i$ to be small. 
For any predictor $\vv$ of view $\z$, we defined the operator $T_{CCA} (\vv)$ to return a
corresponding predictor of view $\x$:
\begin{align*}
T_{CCA} (\vv) = \tU \cdot [\lambda_i (\tV^\top \vv)^i ]_{i=1}^{d_x}.
\end{align*}

\subsection{Multi-view distillation}
\label{sec:multiview-distill}

Similar to single-view distillation analyzed in Section~\ref{sec:distill}, we can
perform a two-step distillation process with multi-view data; this setting
is also named \emph{generalized distillation} by~\citet{Lopez_15a}. 
In the first step, we train a
predictor $\hv$ based on labeled  data from the view $\z$. And in
the second step, we train a predictor on labeled data from the view
$\x$, incorporating the soft supervision provided by $\hv$. Note that the
labeled data from each view need not overlap.

As before, we make the assumption that $\hv$ is learned from the
hypothesis class $\calC_v = \cbr{\vv: \norm{\vv}\le B_v}$ by performing ERM, 
with optimistic rate convergence due to low optimal expected loss $L_v^* = \min_{\vv
  \in \calC_v}\, L(\vv)$.
But different from single-view distillation, since the view $\z$ contains
privileged information (rich representation of the data), learning is easy and a low-complexity hypothesis
class (\eg, $B_v$ is small) suffices. 
The different assumptions on the complexities of hypothesis spaces for distillation vs.
LUPI are motivated by~\citet[Section~4]{Lopez_15a}.

Additionally, assume the optimal predictor $\w^*$ from the hypothesis class $\calC_w =
\cbr{\w: \norm{\w}\le B_w}$ agrees well with $\hv$ on decision values,
\ie, there exists some small $S>0$, such that 
\begin{align} \label{e:agreement-assumption-multiview}
\bbE \norm{{\w^*}^\top \x - \hv^\top \z}^2\le S^2. 
\end{align}
The following lemma provides an example when this assumption is satisfied.
\begin{lem}\label{lem:agreement-multiview}
Let the instantaneous loss $\ell(\cdot, \cdot)$ be $\sigma$-strongly convex
in the first argument. 
Let $\vv$ be a predictor from $\calC_v$ with $B_v \le B_w$, and the
CCA-transformed predictor achieves expected
loss $L_{CCA} (\vv) = \bbE \sbr{\ell \rbr{(T_{CCA}    (\vv))^\top \x, y}}$. 
Then we have
\begin{gather*}
\norm{\w^* - T_{CCA}  (\vv) }^2 = \bbE \norm{{\w^*}^\top \x -
  (T_{CCA}  (\vv))^\top \x }^2 \le \frac{2 (L_{CCA} (\vv) -  L_w^*) }{\sigma}, \\
\text{and}\qquad\  \bbE \norm{{\w^*}^\top \x - \vv^\top \z }^2 \le
\frac{2(L_{CCA} (\vv) -  L_w^*)}{\sigma} +  (1-\lambda_{d_z}) B_v^2.
\end{gather*}
where the expectation is taken over the data distribution.
\end{lem}
\begin{align*}
\end{align*}
\begin{proof}
The proof is similar to that of Lemma~\ref{lem:agreement}. 
By the definition of $T_{CCA}$, and the fact that $\lambda_i \le 1$, we
have that $\norm{T_{CCA} (\hv)} \le \norm{\hv} \le B_v \le B_w$.
Now that both $\w^*$ and $T_{CCA} (\vv)$ are in $\calC_w$ and $\w^*$ is
the minimizer of $L(\w)$ in $\calC_w$, we can use
the same argument in Lemma~\ref{lem:agreement}  and have
\begin{align*}
L_{CCA}(\vv) - L (\w^*) \ge \frac{\sigma}{2}
  \bbE \norm{(T_{CCA}(\vv))^\top \x - {\w^*}^\top \x }^2.
\end{align*}
By our assumption of $\Exx=\I$, it holds that $\bbE \norm{(T_{CCA}(\vv))^\top \x -
  {\w^*}^\top \x }^2 = \norm{T_{CCA}  (\vv) - \w^*}^2$ and the first inequality follows.

To show the second inequality, just observe from
Lemma~\ref{lem:agreement-as-regularizer} that 
\begin{align*}
\bbE \norm{\w^\top \x - \vv^\top \z} & = \norm{\w - T_{CCA}(\vv)}^2
  +\sum_{i=1}^{d_z} (1-\lambda_i)^2 (\tv^i)^2 \\
& \le \norm{\w - T_{CCA}(\vv)}^2
  + (1-\lambda_{d_z})^2 \sum_{i=1}^d (\tv^i)^2 \\
&  \le \norm{\w - T_{CCA}(\vv)}^2
  + (1-\lambda_{d_z})^2 \norm{\vv}^2.
\end{align*}
\end{proof}

We can then use the soft targets in the regularization, and the following
Theorem is completely analogous to Theorem~\ref{thm:distill}.

\begin{thm} \label{thm:multiview-distill}
Compute either of the following predictors 
\begin{align} 
\label{e:distill-rerm-multiview-1}
& \hw = \argmin_{\w \in \calC_w}\; \frac{1}{n} \sum_{i=1}^n \ell (\w^\top
  \x_i, y_i)  + \frac{\nu}{2} \norm{\w - T_{CCA}(\hv)}^2 \\
\label{e:distill-rerm-multiview-2}
\text{or}\qquad
& \hw = \argmin_{\w \in \calC_w}\; \frac{1}{n} \sum_{i=1}^n \ell (\w^\top
  \x_i, y_i)  + \frac{\nu}{2} \bbE_{\x,\z} \norm{\w^\top \x - \hv^\top \z}^2
\end{align}
with $\nu=\frac{8\beta R^2}{n} + \sqrt{\frac{64\beta^2 R^4}{n^2} +
  \frac{16\beta R^2 L_w^*}{n S^2}}$. 
Then we have
\begin{align*}
\bbE \sbr{L(\hw) - L_w^*} \le \frac{16 \beta R^2 S^2}{n} +
  \sqrt{\frac{16\beta R^2 S^2 L_w^*}{n}}.
\end{align*}
\end{thm}
The difference in the two predictors is
that the regularization term in~\eqref{e:distill-rerm-multiview-2} does
not require computing $T_{CCA} (\hv)$ or even the CCA system. But in view
of Lemma~\ref{lem:agreement-multiview}, the difference in the two
regularizers in~\eqref{e:distill-rerm-multiview-1} and
~\eqref{e:distill-rerm-multiview-2} is bounded by a small constant.

\begin{rmk}
\label{rmk:multiview-distill}
Similar to the single-view distillation setting, the additional complexity
constraint from the multi-view agreement assumption allows us to
effectively reduce the search space when $S^2 \ll B_w^2$. 
In particular, when the assumptions of Lemma~\ref{lem:agreement-multiview}
holds, the predictor~\eqref{e:distill-rerm-multiview-1}
achieves the convergence rate
\begin{align*}
L (\hw) - L_w^* = \calO \rbr{ \frac{\beta R^2 (L_{CCA}(\hv) -
  L_w^*)}{\sigma n} + \sqrt{\frac{\beta R^2 L_w^* (L_{CCA}(\hv) -
  L_w^*)}{\sigma n}} }
\end{align*}
To achieve $\epsilon$ excess error, the sample complexity can be written
as
\begin{align*}
n = \calO \rbr{ \frac{\beta R^2 (L_{CCA}(\hv)-L_w^*)}{\sigma \epsilon}
  \rbr{\frac{\epsilon + L_w^*}{\epsilon}}   }.
\end{align*}
Assuming $L_{CCA} (\hv)-L_w^* = \calO (\epsilon^{\theta})$ for some
$\theta \in (0,\,1)$, then as long as $L_w^* = \calO (\epsilon^{1-\theta})$, an
overall $\calO (\frac{1}{\epsilon})$ sample complexity or $\calO
(\frac{1}{n})$ convergence rate is achieved. Again, the optimistic rate is
achieved with less stringent requirement on $L_w^*$.

Here, the expected loss of $T_{CCA}(\hv)$ measures the quality of the
predictor transformed from a good predictor $\hv$ of view $\z$, and in a
sense how
well the privileged information can be transferred. If $T_{CCA}(\hv)$
is already near-optimal for view $\x$, all the student has to do is to be close to it.
\end{rmk}

\subsection{Simultaneous learning of teacher and student}
\label{sec:multiview-multitask}

In previous sections, we require first learning the teacher
before learning the student. We now show that it is
possible to learn both predictors at the same time, using a multi-task
learning objective. This is the same setup considered by the SVM+ algorithm~\citep{VapnikIzmail15a}.

Similar to the distillation setting, we would like to learn predictors
$(\w,\,\vv)$ with low norm and agree well on
decision values; the hypothesis class is defined as 
\begin{align*}
\calC = \cbr{(\w,\vv): \norm{\w} \le B_w, \; 
\norm{\vv} \le B_v,\; 
\bbE \norm{{\w}^\top \x - {\vv}^\top\z}^2 \le S^2}.
\end{align*}
Let $(\w^*,\, \vv^*) = \argmin_{(\w,\vv) \in \calC}\, L(\w) + L(\vv)$ be
the optimal predictors from $\calC$. Denote $L^* = L(\w^*) + L(\vv^*)$. 

As before, we can learn the hypothesis class with regularized ERM.
 The term $\bbE\norm{\w^\top
    \x - \vv^\top \z}^2$ now regularizes the learning of both $\w$ and $\vv$,
  and this is known as ``co-regularization''~\citep{Sindhw_05b, RosenbBartlet07a}.

\begin{lem} \label{lem:multitask}
Let $\cbr{(\x_i, \z_i, y_i)}_{i=1}^n$ be $n$ i.i.d. samples from the 
distribution $P(\x,\z,y)$. Compute the following
predictors
\begin{align} 
(\hw, \hv) = & \argmin _{\w,\vv}\, \frac{1}{n}  \sum_{i=1}^n \ell (\w^\top
               \x_i, y_i) + \frac{1}{n} \sum_{j=1}^n \ell (\vv^\top \z_j,
               y_j)  \nonumber \\  \label{e:multitask}
             & \qquad + \frac{\gamma}{2}
  \norm{\w}^2  + \frac{\gamma}{2} \norm{\vv}^2 + \frac{\nu}{2} \bbE
  \norm{\w^\top \x - \vv^\top \z}^2
\end{align}
with $\gamma = \frac{1}{B_w^2 + B_v^2} \rbr{ \frac{8 \beta R^2 D^2}{n} +
\sqrt{\frac{64 \beta^2 R^4 D^4}{n^2} +\frac{8 \beta R^2 D^2 L^*}{n}} }$ and $\nu=
\frac{1}{S^2} \rbr{ \frac{8 \beta R^2 D^2}{n} + \sqrt{\frac{64\beta^2 R^4 D^4}{n^2}
  +\frac{8 \beta R^2 D^2 L^*}{n}} }$, where $D^2 := \frac{(B_w^2 + B_v^2) S^2 (B_w^2 + B_v^2 + S^2)}{(B_w^2 + B_v^2 + S^2)^2 -
\lambda_1^2 (B_w^2 + B_v^2)^2 }$. 
Then we have
\begin{align*}
\bbE [(L (\hw) + L(\hv)) - L^*] \le 
\frac{32 \beta R^2 D^2}{n} + \sqrt{\frac{32 \beta R^2 D^2 L^*}{n}}. 
\end{align*}
\end{lem}

\begin{proof}
Perform the change of coordinates in~\eqref{e:cca-coordinate-system}, and
let $\p = \sbr{\begin{array}{c}\tw\\\tv\end{array}} \in \bbR^{d_x+d_z}$ be the concatenated variable. 
In view of Lemma~\ref{lem:agreement-as-regularizer}, we observe that the
regularization in~\eqref{e:multitask} (sum of the last three terms) is 
\begin{align*}
Reg(\p) =  \frac{\gamma}{2} \p^\top \sbr{
\begin{array}{cc}
 \I &  \\
 &  \I
\end{array}}
\p  + \frac{\nu}{2} \p^\top \sbr{
\begin{array}{cc}
\I & - \Sigma \\
- \Sigma^\top & \I
\end{array}}
\p 
\end{align*}
where the matrix $\Sigma$ is the diagonal matrix in the SVD of $\Exz$. 

Let $\M = \sbr{
\begin{array}{cc}
\gamma \I &  \\
 & \gamma \I
\end{array}} + 
 \nu \sbr{\begin{array}{cc}
\I & - \Sigma \\
- \Sigma^\top & \I
\end{array}}$, and so $Reg (\p) = \frac{1}{2} \p^\top \M \p$. 
Further define the change of variable $\q = \M^{\frac{1}{2}} \p$, 
and we have that the regularizer $Reg (\p)=\frac{1}{2} \norm{\q}^2$ is $1$-strongly convex in $\q$.

On the other hand, we can rewrite the instantaneous loss in terms of $\q$:
\begin{align*}
\ell (\w^\top \x, y) = \ell (\tw^\top \tx, y)
= \ell(\p^\top \sbr{\begin{array}{c}\tx\\\0\end{array}}, y)  
= \ell(\q^\top \M^{-\frac{1}{2}} \sbr{\begin{array}{c}\tx\\\0\end{array}}, y)  .
\end{align*}
The smoothness of the loss with respect to $\q$ depends on the size of its
input, which is 
\begin{align} \label{e:multitask-input-size}
\norm{ \M^{-\frac{1}{2}} \sbr{\begin{array}{c}\tx\\\0\end{array}} }^2
\le 
\norm{(\M^{-1})_{xx}} \cdot \norm{\tx}^2
\le \norm{(\M^{-1})_{xx}} R^2
\end{align}
where $(\M^{-1})_{xx}$ is the top left block of $\M^{-1}$ with dimension
$d_x \times d_x$.
By the inversion formula of block matrices, we have
\begin{align*}
(\M^{-1})_{xx} = \rbr{ (\gamma + \nu) \I - \frac{\nu^2}{\gamma +
  \nu} \Sigma \Sigma^\top }^{-1}.
\end{align*}
Since the minimum eigenvalue of the matrix inside $\rbr{}$ is at least 
$(\gamma + \nu) - \frac{\lambda_1^2 \nu^2}{\gamma + \nu}$, we continue
from~\eqref{e:multitask-input-size} and claim that the instantaneous loss
of view $\x$ is smooth in $\q$ with parameter
\begin{align*}
\beta^\prime = \frac{ (\gamma + \nu) \beta  R^2}{(\gamma + \nu)^2 - \lambda_1^2 \nu^2}.
\end{align*}
The same smoothness condition holds for the view $\z$, and therefore the
combined instantaneous loss $\ell(\w^\top \x, y) + \ell(\vv^\top \z, y)$ is $\beta^\prime$-smooth in
$\q$. 

The empirical loss in~\eqref{e:multitask} thus approximate the expectation of
the combined loss with $n$ paired samples. 
By the stability of regularized ERM for smooth loss, we have that 
\begin{align*}
\bbE [L (\hw) + L(\hv)] \le \frac{1}{1 - \frac{8 \beta^\prime}{n}} \bbE
  [\hL (\hw) + \hL (\hv)].
\end{align*}
Since $(\hw,\,\hv)$ is the minimizer of the multi-task objective, we have
\begin{align*}
\hL (\hw) + \hL (\hv) 
& \le \hL (\w^*) + \hL (\vv^*) + \frac{\gamma}{2}
  \norm{\w^*}^2 + \frac{\gamma}{2} \norm{\vv^*}^2  + \frac{\nu}{2}
  \bbE \norm{{\w^*}^\top \x - {\vv^*}^\top \z}^2 \\
& \le \hL (\w^*) + \hL (\vv^*) + \frac{\gamma (B_w^2 + B_v^2)}{2} + \frac{\nu S^2}{2}.
\end{align*}
Define the shorthand $B^2=B_w^2 + B_v^2$. 

Taking the expectation over the $n$ samples and re-organizing terms, we
obtain
\begin{align*}
\bbE [(L (\hw) + L(\hv)) - L^*] \le \rbr{ \frac{1}{1 -
  \frac{8 \beta^\prime}{n}} -1 }  L^* + \frac{1}{1 -
  \frac{8 \beta^\prime}{n}} \rbr{\frac{\gamma B^2}{2} + \frac{\nu S^2}{2}}.
\end{align*}
Let $\gamma = \frac{\Delta}{B^2}$ and $\nu = \frac{\Delta}{S^2}$. 
Substituting them into the definition of $\beta^\prime$, we obtain
\begin{align*}
\beta^\prime = \frac{\alpha}{\Delta} 
\qquad\text{where}\quad
\alpha = \frac{\beta R^2 B^2 S^2 (B^2 + S^2)}{(B^2 + S^2)^2 - \lambda_1^2 B^4}.
\end{align*}
Furthermore, 
\begin{align*}
\bbE [(L (\hw) + L(\hv)) - L^*] \le \rbr{ \frac{1}{1 -
  \frac{8 \alpha}{n \Delta}} -1 } L^*+ \frac{1}{1 -
  \frac{8 \alpha}{n \Delta}} \Delta.
\end{align*}
Minimizing the RHS over $\Delta$, we obtain 
\begin{align*}
\Delta=\frac{8 \alpha}{n} + \sqrt{\frac{64\alpha^2}{n^2} +\frac{8\alpha L^*}{n}}
\end{align*}
and consequently
\begin{align*}
\bbE [(L (\hw) + L(\hv)) - L^*] \le 
\frac{16 \alpha}{n} + 2 \sqrt{\frac{64\alpha^2}{n^2} +\frac{8\alpha
  L^*}{n}}
\le \frac{32 \alpha}{n} + \sqrt{\frac{32 \alpha L^*}{n}}. 
\end{align*}
\end{proof}
The theorem assumed that we have access to paired two-view data
with corresponding label. This is mainly for simplicity of presentation, 
and the analysis can be extended to non-paired labeled data and different
amount of labeled data.

\begin{rmk}
\label{rmk:multitask}
Assume that $\lambda_1 < 1$, \ie, the views are not perfectly correlated,
and $S^2 \ll B^2$, then we have in Lemma~\ref{lem:multitask} that 
\begin{align*}
D^2 \approx \frac{S^2}{1 - \lambda_1^2}
\end{align*}
and thus the complexity of the hypothesis class is effectively controlled by
$S^2$ only and not $B_w^2$ or $B_v^2$ (this can also be seen in the
objective~\eqref{e:multitask}, where the regularization parameters satisfy
$\gamma \ll \nu$),  and we obtain similar optimistic rate as in the multi-view
distillation setting. 
Note that the top canonical correlation comes into play here, fundamentally because
if $\lambda_1=1$, the regularizer $\frac{1}{2} \bbE \norm{\w^\top \x - \vv^\top \z}$
is not strongly convex in $\w$ and $\vv$ jointly (even though it is
$1$-strongly convex in each set of variable separately).
\end{rmk}

\section{Related work}
\label{sec:related}

We briefly review prior analysis that is most relevant to ours.

Under the LUPI setting, \citet{PechyonVapnik10a} focused on supervised learning with the $0$-$1$ loss, and proposed SVM based methods that can achieve faster convergence than without using privileged information; in their case, the complexity of the hypothesize class is controlled by its VC-dimension. 

For learning with the least squares loss, \citet{KakadeFoster07a} proposed to use  the predictor $\sbr{ \frac{\lambda_i}{n} \sum_{j=1}^n \x_j y_j }$ (dubbed the ``canonical shrinkage estimator'') for view $\x$, where $\lambda_i$'s are the canonical correlations between $\x$ and $\z$, which corresponds to the solution to the ERM objective regularized by a special ``canonical norm''.\footnote{But strictly speaking, it is not a norm if any $\lambda_i=0$ or $1$.} 
The authors showed that the shrinkage estimator has lower variance than the non-regularized ERM solution, while having similar bias. Their analysis relied heavily on the closed-form solution for least squares.

For jointly learning two (kernel-based) predictors $f(\x)$ and $g(\x)$ in a single view, \citet{RosenbBartlet07a} studied the Rademacher complexity of the function class $\calC=\cbr{(f,g): \gamma_f \norm{f}^2 + \gamma_g \norm{g}^2 + \nu \norm{f (\x) - g (\x)}^2 \le 1}$, and observed that as the tradeoff parameter for $\nu$ increases, the Rademacher complexity of $\calC$ decreases. This result is consistent with our Remark~\ref{rmk:multitask}.

In understanding how unlabeled data help with supervised learning, \citet{BalcanBlum10a} showed that suitably defined compatibility notion on the unlabeled data can reduce the search space for learning the optimal predictor. In the case of multi-view co-training~\citep{BlumMitchel98a}, the compatibility requires the binary predictions from both views to agree on unlabeled data. Under the assumption that views are independence given the label, the authors provide an efficiently algorithm for training linear separators from a single labeled example. 
In this work, we share the same intuition that compatibility between the views reduces the search space for learning, and propose to use the easy-to-optimize regularizer $\bbE \norm{\w^\top \x - \vv^\top \y}$ in the generalized PAC learning setting, which can be seen as a soft version of their compatibility constraint. 


\bibliographystyle{plainnat}

\bibliography{macp,macp-xref}
\clearpage
\appendix
\section{Stability for learning with smooth and non-negative loss}
\label{sec:stability}

An important property of smooth and non-negative loss is the self-boundedness~\citep[Section~12.1.3]{ShalevBen-David14a}: if $f(\w) \ge 0$ is convex and $\beta$-smooth, we have
\begin{align*}
  \norm{ \nabla f(\w) }^2 \le 2 \beta f (\w).
\end{align*}

The following stability of ERM follows~\citet[Lemma~4.2]{Srebro_12a} (with
slightly better constants than theirs), we provide the proof for completeness.
\begin{lem}[Stability for convex, smooth non-negative loss] 
\label{lem:stability}
  Consider a stochastic convex optimization problem 
  \begin{align*}
    F (\w) := \bbE_{\xi} \left[f (\w; \xi) \right]  + r (\w)
  \end{align*}
  where the loss $f (\w; \xi) \ge 0$ is $\beta$-smooth and $\lambda$-strongly convex in $\w$, the regularization $r (\w)$ is $\gamma$-strongly convex. Let $Z=\{ \xi_1,\dots,\xi_n \}$ be i.i.d. samples and 
  \begin{align*}
    \hw = \argmin_{\w \in \Omega} \ \hF (\w) \qquad \text{where}\quad \hF (\w) := \frac{1}{n} \sum_{i=1}^n f (\w; \xi_i)  + r (\w).
  \end{align*}
  Denote by $G (\w) := \bbE_{\xi} \left[f (\w; \xi) \right]$ and $\hG (\w) := \frac{1}{n} \sum_{i=1}^n f (\w; \xi_i)$ the expected loss and the empirical loss respectively. 
 For the regularized empirical risk minimizer $\hw$, we have for $(\lambda+\gamma) n \ge 8 \beta$ that 
    \begin{align*}
      \left( 1 - \frac{8 \beta}{(\lambda+\gamma) n} \right) \cdot \bbE_{Z} \left[ G (\hw) \right] \le  \bbE_{Z} \left[ \hG (\hw) \right].
    \end{align*}
\end{lem}

\begin{proof}
 Denote by $Z^{(i)}$ the sample set that is identical to $Z$ except that the $i$-th sample $\xi_i$ is replaced by another random sample $\xi_i^\prime$, by $\hF^{(i)} (\w)$ the empirical objective defined using $Z^{(i)}$, \ie,
  \begin{align*}
    \hF^{(i)} (\w) := \frac{1}{n} \left ( \sum_{j\neq i} f (\w; \xi_i) +  f (\w; \xi_i^\prime) \right) + r (\w),
  \end{align*}
  and by $\hw^{(i)}=\argmin_{\w \in \Omega}\ \hF^{(i)} (\w)$ the empirical risk minimizer of $\hF^{(i)} (\w)$.

  By the definition of the empirical objectives, we have
  \begin{align} 
    \hF (\hw^{(i)}) - \hF (\hw) 
    & = \frac{f (\hw^{(i)}; \xi_i) - f (\hw; \xi_i) }{n} 
    + \frac{\sum_{j\neq i} f (\hw^{(i)}; \xi_i) - f (\hw; \xi_i) }{n} 
    + r (\hw^{(i)}) - r (\hw) \nonumber \\
    & = \frac{f (\hw^{(i)}; \xi_i) - f (\hw; \xi_i) }{n} 
    + \frac{ f (\hw; \xi_i^\prime) - f (\hw^{(i)}; \xi_i^\prime) }{n} 
    + \left( \hF^{(i)} (\hw^{(i)}) - \hF^{(i)} (\hw) \right) \nonumber  \\ 
    \nonumber 
    & \le \frac{ f (\hw^{(i)}; \xi_i) - f (\hw; \xi_i) }{n} 
    + \frac{ f (\hw; \xi_i^\prime) - f (\hw^{(i)}; \xi_i^\prime) }{n}
  \end{align}
  where we have used the fact that $\hw^{(i)}$ is the minimizer of $\hF^{(i)} (\w)$ in the inequality.

  On the other hand, it follows from the $(\lambda + \gamma)$-strong convexity of $\hF (\w)$ that 
  \begin{align*} 
    \hF (\hw^{(i)}) - \hF (\hw) \ge \frac{(\lambda + \gamma)}{2} \norm{ \hw^{(i)} - \hw }^2,
  \end{align*}
  and therefore
  \begin{align} \label{e:stability-strong-convexity-smoothness}
    \frac{(\lambda + \gamma) n }{2} \norm{ \hw^{(i)} - \hw }^2 \le \left( f (\hw^{(i)}; \xi_i) - f (\hw; \xi_i) \right) + \left( f (\hw; \xi_i^\prime) - f (\hw^{(i)}; \xi_i^\prime) \right).
  \end{align}

  By the $\beta$-smoothness and self-boundedness property, we have
  \begin{align}
    & \left( f (\hw^{(i)}; \xi_i) - f (\hw; \xi_i) \right) + \left(  f (\hw; \xi_i^\prime)  - f (\hw^{(i)}; \xi_i^\prime) \right) \nonumber \\
    & \le \left< \nabla f(\hw^{(i)}; \xi_i),\, \hw^{(i)} - \hw \right> 
    + \left< \nabla f(\hw; \xi_i^\prime),\, \hw - \hw^{(i)} \right> 
    \nonumber \\
    & \le \left(\norm{\nabla f(\hw^{(i)}; \xi_i)} + \norm{\nabla f(\hw; \xi_i^\prime)} \right) \norm{\hw^{(i)} - \hw}  \nonumber \\  \label{e:stability-smoothness-one-sample-difference}
    & \le \left( \sqrt{2 \beta f(\hw^{(i)}; \xi_i)} + \sqrt{2 \beta f(\hw; \xi_i^\prime)} \right) \norm{\hw^{(i)} - \hw}.
  \end{align}

  Combining~\eqref{e:stability-smoothness-one-sample-difference} and~\eqref{e:stability-strong-convexity-smoothness} yields
  \begin{align*}
    \norm{\hw^{(i)} - \hw}  \le \frac{\sqrt{8 \beta}}{(\lambda+\gamma) n} \left(  \sqrt{f(\hw^{(i)}; \xi_i)} + \sqrt{ f(\hw; \xi_i^\prime)}  \right).
  \end{align*}

  Plugging this into~\eqref{e:stability-smoothness-one-sample-difference}, we obtain
  \begin{align*} 
    \left( f (\hw^{(i)}; \xi_i) - f (\hw; \xi_i) \right) + \left(  f (\hw; \xi_i^\prime)  - f (\hw^{(i)}; \xi_i^\prime) \right)
    & \le \frac{4\beta}{(\lambda+\gamma) n} \left(  \sqrt{f(\hw^{(i)}; \xi_i)} + \sqrt{ f(\hw; \xi_i^\prime)}  \right)^2 \\
    & \le \frac{8\beta}{(\lambda+\gamma) n} \left( f(\hw^{(i)}; \xi_i) + f(\hw; \xi_i^\prime) \right)
  \end{align*}
  where we have used the fact that $(a+b)^2 \le 2 (a^2 + b^2)$ in the last inequality. 

  We obtain the desired result by taking expectation of the above inequality, and realizing that 
  \begin{align*}
    \bbE_{Z\cup \{\xi_i^\prime\}} \left[f (\hw^{(i)}; \xi_i) \right] = \bbE_{Z\cup \{\xi_i^\prime\}} \left[f (\hw; \xi_i^\prime)\right] = \bbE_{Z} \left[ G (\hw) \right],\\
    \bbE_{Z} \left[ f(\hw; \xi_i) \right] = \bbE_{Z^{(i)}} \left[ f(\hw^{(i)}; \xi_i^\prime) \right]  =  \bbE_{Z} \left[ \hG (\hw) \right].
  \end{align*}

\end{proof}

\end{document}